\documentclass[a4paper,10pt]{article}
\usepackage[utf8]{inputenc}
\usepackage[margin=1in]{geometry}
\pdfoutput=1

\usepackage{amsmath, amsfonts, amssymb, amsthm}
\usepackage{xfrac, multicol}
\usepackage{subfigure}

\usepackage[usenames,dvipsnames,svgnames,table]{xcolor}

\usepackage{algorithm}
\usepackage{algorithmicx}
\usepackage{algpseudocode}

\usepackage{tikz}
\usepackage{pgfplots}
\usetikzlibrary{calc, patterns, decorations.pathmorphing,
decorations.markings, intersections, arrows, patterns, plotmarks, spy, shadows}
\pgfplotsset{compat=newest}
\usepgfplotslibrary{patchplots}
\tikzset{spy using overlaysshadow/.style={
    spy scope={#1,
         every spy on node/.style={
            circle,
            fill, fill opacity=0.2, text opacity=1
            },
         every spy in node/.style={
                 circle, circular drop shadow,
                 fill=white, draw, ultra thick, cap=round
            }
        }
    }
}

\usepackage{amsthm}
\usepackage{thmtools}
\declaretheoremstyle[
notefont=\bfseries, notebraces={}{},
bodyfont=\itshape,
postheadspace=0.5em,
numbered=no,
]{mystyle}
\declaretheorem[style=mystyle]{theorem}

\declaretheorem[style=mystyle]{lemma}

\theoremstyle{remark}

\theoremstyle{definition}

\newif\ifproof
\prooftrue

\newcommand{\PMGA}[1][~]{PMGA#1}

\newcommand{\transpose}[1]{{#1}^\texttt{T}}

\newcommand{\norm}[2][\infty]{\left\|#2\right\|_{#1}}

\newcommand{\realspace}{\mathbb R}		

\newcommand{\statespace}{\mathcal S}		
\newcommand{\pmodel}{\mathcal P}		
\newcommand{\rmodel}{\mathcal R}		

\newcommand{\jvect}{\mathbf{J}}			

\newcommand{\frontier}{\mathcal{F}}
\newcommand{\curveparam}{\boldsymbol{\rho}}

\newcommand{\pp}{\theta} 			
\newcommand{\ppvect}{\boldsymbol{\pp}}		
\newcommand{\dparam}{d}

\newcommand{\vecop}{\text{vec }}

\newcommand{\Jval}[1][\ppvect]{J\left({#1}\right)}

\newcommand{\dstate}{n}
\newcommand{\daction}{m}
\newcommand{\dobj}{q}
\newcommand{\dpolicy}{d}
\newcommand{\dmandim}{b}

\newcommand{\traj}{\tau}
\newcommand{\trajset}{\mathbb{T}}

\newcommand{\pdfunc}[1]{p\left(#1\right)}

\newcommand{\determinant}[1]{det\left({#1}\right)}

\newcommand{\JacobJ}[1][]{D_{\ppvect}\jvect_{{#1}}(\ppvect)}

\newcommand{\Hessian}{H}
\newcommand{\HJ}[1][i]{\Hessian_{\ppvect}\jvect_{{#1}}(\ppvect)}

\newcommand{\HJhat}[1][i]{\widehat{\Hessian}_{\ppvect}\jvect_{{#1}}(\ppvect)}

\newcommand{\Rbound}{\overline{R}}
\newcommand{\Dbound}{\overline{D}}
\newcommand{\Hbound}{\overline{G}}
\newcommand{\trajlength}{H}

\newcommand{\Indf}{\mathcal{I}}

\newcommand{\vol}[1]{Vol\left( {#1} \right)}
\newcommand{\polmap}{\phi_{\curveparam}}
\newcommand{\variable}{\mathbf{t}}
\newcommand{\polmapspace}{\mathcal{T}}

\newcommand{\TJP}{\mathbf{T}}

\title{Multi--objective Reinforcement Learning with Continuous Pareto Frontier Approximation Supplementary Material}
\author{Matteo Pirotta \and Simone Parisi \and Marcello Restelli\\[.1cm]
\normalsize
Department of Electronics, Information and Bioengineering, Politecnico di Milano,\\
\normalsize
Piazza Leonardo da Vinci, 32, 20133, Milan, Italy\\
\normalsize
matteo.pirotta@polimi.it,
simone.parisi@mail.polimi.it,
marcello.restelli@polimi.it
}
\date{}

\begin{document}

\maketitle
\begin{abstract}
This document contains supplementary material for the paper ``Multi--objective Reinforcement Learning with Continuous Pareto Frontier Approximation'', 
published at the \textit{Twenty--Ninth AAAI Conference on Artificial Intelligence (AAAI-15)}.
The paper is about learning a continuous approximation of the Pareto frontier in Multi-Objective Markov Decision Problems (MOMDPs). We propose a policy-based approach that exploits gradient information to generate solutions close to the Pareto ones. Differently from previous policy-gradient multi-objective algorithms, where n optimization routines are use to have n solutions, our approach performs a single gradient-ascent run that at each step generates an improved continuous approximation of the Pareto frontier. The idea is to exploit a gradient-based approach to optimize the parameters of a function that defines a manifold in the policy parameter space so that the corresponding image in the objective space gets as close as possible to the Pareto frontier. Besides deriving how to compute and estimate such gradient, we will also discuss the non-trivial issue of defining a metric to assess the quality of the candidate Pareto frontiers. Finally, the properties of the proposed approach are empirically evaluated on two interesting MOMDPs.
\end{abstract}

The paper ``Multi--objective Reinforcement Learning with Continuous Pareto Frontier Approximation'' has been published at the Twenty--Ninth AAAI Conference on Artificial Intelligence (AAAI-15). 
This supplement follows the same structure of the main article. For each section we report the complete set of proofs and some additional details. 

\section{Reparametrization}
In this section we provide the proof of an extended version of Theorem 1.
\begin{theorem}[1] 
  Let $\polmapspace$ be an open set in $\realspace^{\dmandim}$, let $\frontier_{\curveparam}\left( \polmapspace \right)$ be a manifold parametrized by a smooth map 
expressed as composition of maps $\jvect$ and $\polmap$, ($\jvect \circ \polmap :\polmapspace \to  \realspace^{\dobj}$).
Given a continuous function $\Indf$ defined at each point of $\frontier_{\curveparam}(\polmapspace)$, the integral w.r.t. the volume is given by
 \begin{align*}
J(\curveparam) 
 &= \int_{\frontier(\polmapspace)}\Indf \mathrm{d}V 
  = \int_{\polmapspace} \left( \Indf \circ \left( \jvect \circ \polmap \right) \right) \vol{ D_{\ppvect}\jvect(\ppvect) D_{\variable}\polmap(\variable) } \mathrm{d}\variable.
\end{align*}
The associated gradient w.r.t. the map parameters $\curveparam$ is given component--wise by
\begin{align*}
\frac{\partial \Jval[\curveparam]}{\partial \curveparam_i} 
 &= \int_{\polmapspace} \frac{\partial}{\partial \curveparam_i} \left( \Indf \circ \left( \jvect \circ \polmap \right) \right) 
    \vol{ \TJP } \mathrm{d}\variable\\	
 &\quad{} + \int_{\polmapspace} \left( \Indf \circ \left( \jvect \circ \polmap \right) \right)   
       \vol{ \TJP }\transpose{\left(  \vecop\left( \transpose{\TJP}\TJP \right)^{-\texttt{T}}\right)}
       N_{\dmandim} \left( I_{\dmandim} \otimes \transpose{\TJP} \right)
	D_{\curveparam_i} \TJP
       \mathrm{d}\variable
\end{align*}
where $\TJP = D_{\ppvect}\jvect(\ppvect) D_{\variable}\polmap(\variable)$, $\otimes$ is the Kronecker product, $N_{\dmandim} = \frac{1}{2} \left(I_{\dmandim^2} + K_{\dmandim \dmandim}\right)$ is a symmetric $(\dmandim^2 \times \dmandim^2)$ idempotent matrix with rank $\frac{1}{2} \dmandim (\dmandim + 1)$ and $K_{\dmandim \dmandim}$ is a permutation matrix~\cite{magnus1999matrix}.
Note that
\begin{align*}
 D_{\curveparam_i} \TJP
 &= \left( \transpose{D_{\variable}\polmap(\variable)} \otimes I_{\dobj} \right)
 D_{\ppvect} \left( \JacobJ[] \right)
 D_{\curveparam_i} \polmap(\variable) 
 + \left( I_{\dmandim} \otimes D_{\ppvect}\jvect(\ppvect) \right) 
 D_{\curveparam_i} \left( D_{\variable}\polmap(\variable) \right)
\end{align*}
\end{theorem}
\begin{proof}
The equation of the performance measure $\jvect(\rho)$ follows directly from the definition of volume integral of a manifold~\cite{munkres1997analysis} and the definition of function composition.
In the following we give a detailed derivation of the $i$--th component of the gradient.
Let $\TJP = D_{\ppvect}\jvect(\ppvect_{\variable}) D_{\variable}\polmap(\variable)$, then
\begin{align*}
\frac{\partial \Jval[\curveparam]}{\partial \curveparam_i} 
 &= \int_{\polmapspace} \frac{\partial}{\partial \curveparam_i} \left( \Indf \circ \left( \jvect \circ \polmap \right) \right) 
    \vol{ D_{\ppvect}\jvect(\ppvect_{\variable}) D_{\variable}\polmap(\variable) } \mathrm{d}\variable\\	
 &\quad{} + \int_{\polmapspace} \left( \Indf \circ \left( \jvect \circ \polmap \right) \right) 
      \frac{1}{2 \vol{ \TJP } }
      \frac{\partial \determinant{ \transpose{\TJP} \TJP }}{\partial \curveparam_i } \mathrm{d}\variable,
\end{align*}
where the pedix $\variable$ is used to denote the direct or indirect dependence on variable $\variable$. 
While the loss derivative and the determinant derivative can be respectively expanded as
\begin{equation*}
 \frac{\partial}{\partial \curveparam_i} \left( \Indf \circ \left( \jvect \circ \polmap \right) \right) 
 = D_{\jvect} \Indf(\jvect_{\variable}) \cdot D_{\ppvect} \jvect(\ppvect_{\variable}) \cdot D_{\curveparam_i} \polmap(\variable),
\end{equation*}
\begin{align*}
\underbrace{
\frac{\partial \determinant{\transpose{\TJP}\TJP}}{\partial \curveparam_i}
}_{1 \times 1}
 &= 
\underbrace{
   \frac{\partial \determinant{\transpose{\TJP}\TJP}}{\partial \transpose{(\vecop{\TJP})}}
}_{1 \times \dmandim^2}
\underbrace{
 \frac{\partial \vecop \transpose{\TJP}\TJP}{\partial \transpose{(\vecop{\TJP})}}
}_{\dmandim^2 \times \dobj \dmandim}
\underbrace{
 \frac{\partial \TJP}{\partial \curveparam_i}
}_{\dobj \dmandim \times 1}
\end{align*}
where
\begin{align*}
\frac{\partial \determinant{\transpose{\TJP}\TJP}}{\partial \transpose{(\vecop{\TJP})}}
 &= \determinant{\transpose{\TJP}\TJP} \transpose{\left(  \vecop\left( \transpose{\TJP}\TJP \right)^{-\texttt{T}}\right)} \\
\frac{\partial \transpose{\TJP}\TJP}{\partial \transpose{(\vecop{\TJP})}}
 &= 2 N_{\dmandim} \left( I_{\dmandim} \otimes \transpose{\TJP} \right) \\
\end{align*}
and $\otimes$ is the Kronecker product, $N_{\dmandim} = \frac{1}{2} \left(I_{\dmandim^2} + K_{\dmandim \dmandim}\right)$ is a symmetric $(\dmandim^2 \times \dmandim^2)$ idempotent matrix with rank $\frac{1}{2} \dmandim (\dmandim + 1)$ and $K_{\dmandim \dmandim}$ is a permutation matrix~\cite{magnus1999matrix}.

The last term to be expanded is $D_{\curveparam_i} \TJP := \frac{\partial \vecop \left( \TJP \right)}{\partial \curveparam_i}$. 
We star from a basic property of the differential
\begin{align*}
 d \left( D_{\ppvect}\jvect(\ppvect) D_{\variable}\polmap(\variable) \right) 
 &= d (D_{\ppvect}\jvect(\ppvect)) D_{\variable}\polmap(\variable) + D_{\ppvect}\jvect(\ppvect)\; d (D_{\variable}\polmap(\variable))\\
 \intertext{then, applying the vector operator,}
 d \vecop \left( D_{\ppvect}\jvect(\ppvect) D_{\variable}\polmap(\variable) \right)
 &= \vecop \left(d (D_{\ppvect}\jvect(\ppvect)) D_{\variable}\polmap(\variable) \right) + \vecop \left( D_{\ppvect}\jvect(\ppvect)\; d (D_{\variable}\polmap(\variable)) \right)\\
 &= 
 \underbrace{
 \left( \transpose{D_{\variable}\polmap(\variable)} \otimes I_{\dobj} \right)
 }_{\dmandim \dobj \times \dpolicy \dobj} 
 \underbrace{ 
 d \vecop (D_{\ppvect}\jvect(\ppvect))
 }_{\dpolicy \dobj \times 1}
 +
 \underbrace{ 
 \left( I_{\dmandim} \otimes D_{\ppvect}\jvect(\ppvect) \right)
 }_{\dmandim\dobj \times \dmandim\dpolicy}
 \underbrace{
 d\vecop(D_{\variable}\polmap(\variable))
 }_{\dmandim\dpolicy \times 1}
\end{align*}
Finally, the derivative is given by	
\begin{align*}
 D_{\curveparam_i} \TJP
 &= \left( \transpose{D_{\variable}\polmap(\variable)} \otimes I_{\dobj} \right)
 \underbrace{
 \frac{\partial \vecop D_{\ppvect}\jvect(\ppvect)}{\partial \transpose{\ppvect}}
 }_{\dpolicy \dobj \times \dpolicy}
 \underbrace{
 \frac{\partial \polmap(\variable)}{\partial \curveparam_i}
 }_{\dpolicy \times 1}
 + \left( I_{\dmandim} \otimes D_{\ppvect}\jvect(\ppvect) \right) 
 \underbrace{
 \frac{\partial \vecop D_{\variable}\polmap(\variable)}{\partial \curveparam_i}
 }_{\dmandim \dpolicy \times 1}\\
 &= \left( \transpose{D_{\variable}\polmap(\variable)} \otimes I_{\dobj} \right) 
 D_{\ppvect} \left( \JacobJ[]\right)
 D_{\curveparam_i} \polmap(\variable)
 + \left( I_{\dmandim} \otimes D_{\ppvect}\jvect(\ppvect) \right) 
 D_{\curveparam_i} \left( D_{\variable}\polmap(\variable)\right)\\
\end{align*}
Note that $D_{\ppvect} \left( \JacobJ[]\right) = \frac{\partial \vecop \JacobJ[]}{\partial \transpose{\ppvect}}$ do not denote the Hessian  matrix. In fact, the Hessian matrix is defined as the derivative of the transpose Jacobian, that is, $\HJ[] = D_{\ppvect} \transpose{ \left(  \JacobJ[] \right) }$. 
The following equation relates the Hessian matrix to $D_{\ppvect} \left( \JacobJ[]\right)$:
\begin{equation*}
 H^{m,n}_{\ppvect} J_i = D^{2[n,m]}_{\ppvect} \jvect_i(\ppvect) = \frac{\partial}{\partial \ppvect_n} \left( \frac{\partial \jvect_i}{\partial \ppvect_m} \right) = D_{\ppvect}^{p,n} \left( \JacobJ[] \right)
\end{equation*}
where $p = i + \dobj (m-1)$, where $\dobj$ is the number of rows of the Jacobian matrix.
\end{proof}

\begin{theorem}[2]
For any MOMDP, the Hessian $\HJ[]$  of the expected discounted reward $\jvect$  w.r.t. the policy parameters $\ppvect$ is  a  ($\dobj \dpolicy \times \dpolicy$) matrix obtained by stacking the Hessian of each component
\begin{equation*}
 \HJ[] = \frac{\partial}{\partial \transpose{\ppvect}} \vecop \transpose{\left(\frac{\partial \jvect_i(\ppvect)}{\partial \transpose{\ppvect}}\right)}
 =\begin{bmatrix}
          \HJ[1]\\
          \vdots\\
          \HJ[q]
         \end{bmatrix},
\end{equation*}
where
 \begin{align}
 \label{E:hessiantraj}
 \HJ[i] 
 &= \int_{\trajset} \pdfunc{\traj|\ppvect} \mathbf{r}_i(\traj) \left(
 \nabla_{\ppvect} \log \pdfunc{\traj|\ppvect}
 \transpose{ \nabla_{\ppvect} \log \pdfunc{\traj|\ppvect}}
 + 
 D_{\ppvect} \left( \nabla_{\ppvect} \log \pdfunc{\traj|\ppvect} \right)
 \right) \mathrm{d}\traj.
\end{align}
\end{theorem}
\begin{proof}
The Hessian equation follows form the definition of the gradient $\nabla_{\ppvect} \jvect(\ppvect)$
\begin{equation*}
 \nabla_{\ppvect} \jvect_i(\ppvect) = \int_{\trajset} \pdfunc{\traj|\ppvect} \mathbf{r}_i(\traj) \nabla_{\ppvect} \log \pdfunc{\traj|\ppvect} \mathrm{d}\traj,
\end{equation*} 
the log trick and the property that the reward of a trajectory $\traj$ is independent from the policy parametrization.
Let outline the derivation of the Hessian matrix.
\begin{align*}
 \mathrm{d} \nabla_{\ppvect} \jvect_i(\ppvect) = \int_{\trajset}  &\mathbf{r}_i(\traj) \nabla_{\ppvect} \log \pdfunc{\traj|\ppvect}  \mathrm{d} \pdfunc{\traj|\ppvect} 
  +\mathbf{r}_i(\traj) \pdfunc{\traj|\ppvect}  \mathrm{d} \left(\nabla_{\ppvect} \log \pdfunc{\traj|\ppvect} \right) \mathrm{d}\traj,
\end{align*}
Then
\begin{align*}
D_{\ppvect} \left(\nabla_{\ppvect} \jvect_i(\ppvect) \right) &= D_{\ppvect} \left( D_{\ppvect} \jvect_i(\ppvect) \right) = H \jvect_i(\ppvect) \\
 &=
 \int_{\trajset}  \mathbf{r}_i(\traj) \pdfunc{\traj|\ppvect} 
 \left[ 
 \nabla_{\ppvect} \log \pdfunc{\traj|\ppvect}  \transpose{ \left( \nabla_{\ppvect} \log \pdfunc{\traj|\ppvect} \right) }
  +  D_{\ppvect} \left( \nabla_{\ppvect} \log \pdfunc{\traj|\ppvect} \right)
 \right] \mathrm{d}\traj\\
 &=
 \int_{\trajset}  \mathbf{r}_i(\traj) \pdfunc{\traj|\ppvect} 
 \left[ 
 \nabla_{\ppvect} \log \pdfunc{\traj|\ppvect}  \transpose{ \left( \nabla_{\ppvect} \log \pdfunc{\traj|\ppvect} \right) }
  +  H \log \pdfunc{\traj|\ppvect} 
 \right] \mathrm{d}\traj.
\end{align*}

Recall that, since the probability of trajectory $\traj$ under policy $\pi^{\ppvect}$ is given by
\begin{equation*}
\pdfunc{\traj|\ppvect} = \pdfunc{s_1}\prod_{k=1}^{\trajlength} \pmodel(s_{k+1}|s_k,a_k)\pi(a_k|s_k,\ppvect),
\end{equation*}
the following equations hold
\begin{align*}
 \nabla_{\ppvect} \log \pdfunc{\traj|\ppvect} &= \sum_{k=1}^{\trajlength} \nabla_{\ppvect} \log \pi(s_k|s_k,\ppvect),\\
 H_{\ppvect} \log \pdfunc{\traj|\ppvect} &= \sum_{k=1}^{\trajlength} H_{\ppvect} \log \pi(s_k|s_k,\ppvect).
\end{align*}

\end{proof}
\begin{lemma}[4]
Given a parametrized policy $\pi(a|s,\ppvect)$, under the assumption Assumption 3, the $i$--th component of the log--Hessian of the expected return can be bounded by
\begin{equation*}
 \norm[\max]{\HJ[i]}\leq \frac{\Rbound_i \trajlength \gamma^{\trajlength}}{1-\gamma} \left( \trajlength \Dbound^2 + \Hbound \right).
\end{equation*}
\end{lemma}
\begin{proof}
Consider the definition of the Hessian in Equation~\eqref{E:hessiantraj}. Under assumption 3, the Hessian components can be bounded by ($\forall m,n$)
\begin{align*}
 \label{E:hessiantraj}
 \Big | H^{m,n}_{\ppvect} \jvect_i(\ppvect) \Big |
 & = \Bigg | \int_{\trajset} \pdfunc{\traj|\ppvect} \mathbf{r}_i(\traj) 
  \sum_{k=1}^{H}\Bigg(
  \frac{\partial}{\partial \ppvect_m} \log \pi(a_k|s_k,\ppvect)
  \sum_{k'=1}^{H} \frac{\partial}{\partial \ppvect_n} \log \pi(a_{k'}|s_{k'},\ppvect)\\
 & \hspace{7cm} + \frac{\partial^2}{\partial \ppvect_m \partial \ppvect_n}  \log \pi(a_k|s_k,\ppvect)
  \Bigg)\ \Bigg |\\
 & \leq \Rbound_i \sum_{l=1}^{\trajlength} \gamma^{l-1} \cdot
 \sum_{k=1}^{\trajlength} \left(
 \Dbound \sum_{k'=1}^{\trajlength} \Dbound + \Hbound 
 \right)
 = \frac{\Rbound_i \trajlength \gamma^{\trajlength}}{1-\gamma} \left( \trajlength \Dbound^2 + \Hbound \right)
\end{align*}
\end{proof}

\begin{theorem}[5]
Given a parametrized policy $\pi(a|s,\ppvect)$, under the assumption Assumption 3, using the following number of $H$--step trajectories
\begin{equation*}
  N=\frac{1}{2\epsilon_i^2}\left(\frac{\Rbound_i \trajlength \gamma^{\trajlength}}{\left(1-\gamma\right)}
  \left( \trajlength \Dbound^2 + \Hbound \right)\right)^2 \log \frac{2}{\delta}
\end{equation*}
the gradient estimate $\HJhat[i]$ generated by Equation
(1) is such that with probability $1-\delta$:
$$\norm[\max]{\HJhat[i] - \HJ[i]} \leq \epsilon_i.$$
\end{theorem}
\begin{proof}
 Hoeffding's inequality implies that, $\forall m,n$
 \begin{equation*}
  \mathbb{P} \left( \widehat{H}^{m,n}_{\ppvect}\jvect_i(\ppvect) - H^{m,n}_{\ppvect}\jvect_i(\ppvect) \geq \epsilon_i \right)
  \leq 2e^{-\frac{N^2 \epsilon_i^2}{\sum_{i=1}^{N} (b_i-a_i)^2}} = \delta
 \end{equation*}
 Solving the equation for $N$, notice that Lemma 4 provides a bound on each samples, we obtain:
 \begin{align*}
  N=\frac{1}{2\epsilon_i^2}\left(\frac{\Rbound_i \trajlength \gamma^{\trajlength}}{\left(1-\gamma\right)}
  \left( \trajlength \Dbound^2 + \Hbound \right)\right)^2 \log \frac{2}{\delta}.
 \end{align*}

\end{proof}
\section{Experiments}
In this section we present the most relevant experiments conducted on two domains (a Linear-Quadratic Gaussian regulator and a water reservoir) in order to study the behavior of \PMGA[] algorithm with the different loss functions $\Indf$ proposed in the paper. We show the frontiers obtained with the loss functions described in the paper and in addition we present a normalization that takes into account the area of the approximate Pareto frontier. The area $A(\curveparam)$ of a manifold is defined as the volume integral of the unitary function:
$$A(\curveparam) = \int_{\frontier\left(\polmapspace\right)} 1\cdot \mathrm{d}V.$$

In the following we propose two different type of normalization:
\begin{itemize}
\item Using the area $A(\curveparam)$ of the frontier $\frontier(\curveparam)$ as normalization factor: $\Indf_n = \Indf A(\curveparam)^{-\beta}$,
\item Using a convex combination of both the area and the loss function: $\Indf_n = w_1\Indf+w_2A(\curveparam)$ with $w_1+w_2 = 1$. 
\end{itemize}
The idea of these normalizations is that the loss function $\Indf$ should guarantee the accuracy of the solutions obtained (i.e., only non-dominated solutions), while the area $A(\curveparam)$ should provide a complete and uniform covering of the frontier.

In all the following experiments the learning rate was hand-tuned.

\begin{figure}[t]
\centering
\subfigure[\label{F:X}]{
 \includegraphics{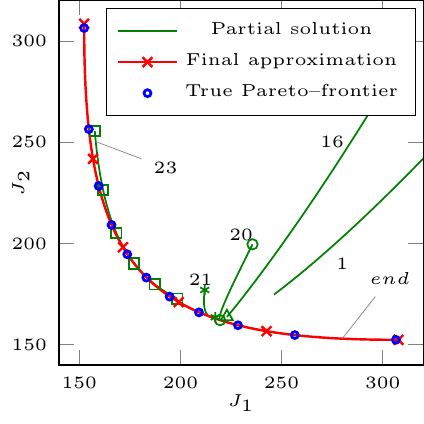}
}
\subfigure[\label{F:Y}]{
 \includegraphics{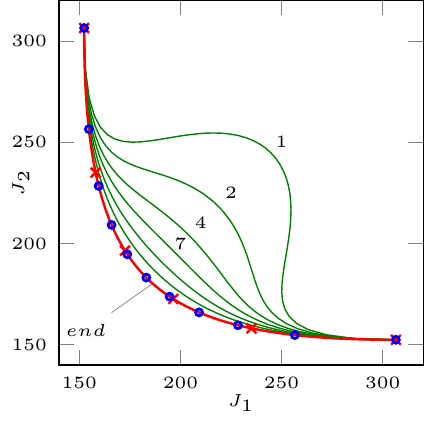}
}
\subfigure[\label{F:Z}]{
 \includegraphics{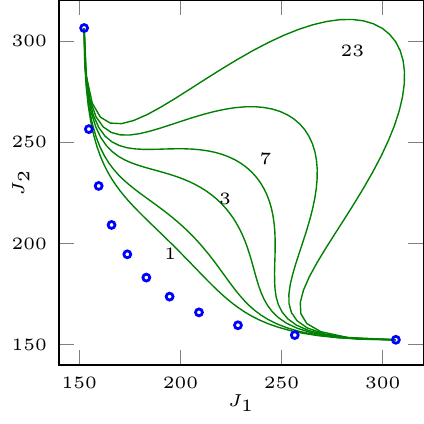}
}
\caption{Learning processes for the 2-objectives LQG (numbers denote the iteration) obtained through \PMGA[]. In Figures \ref{F:Y} and \ref{F:X} $\Indf_3$ is used, respectively with and without forcing the parametrization to pass through extrema. Figure \ref{F:Z} shows iterations with $\Indf_1(\jvect,\mathbf{p}_{au})$.}
\end{figure}

\begin{figure}[t]
\centering
\subfigure[\label{F:X-Jr}Using $\Indf_3$ the learning converges.]{
 \includegraphics{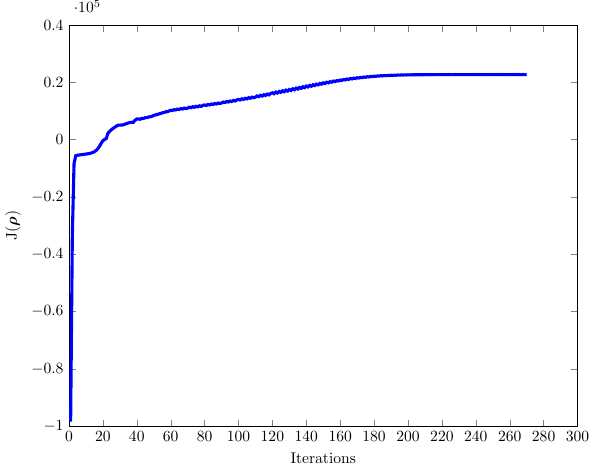}
}
\subfigure[\label{F:Z-Jr}Using $\Indf_1$ the learning diverges.]{
 \includegraphics{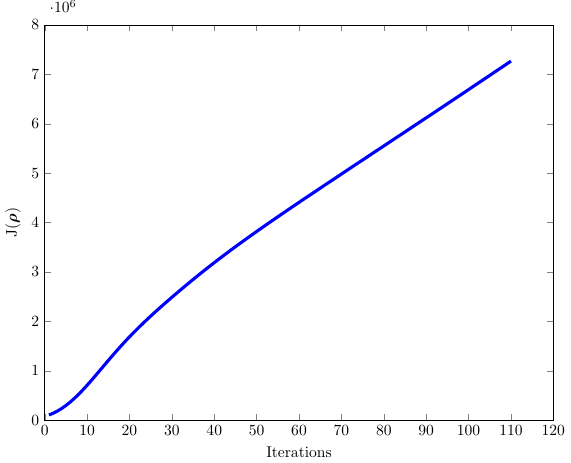}
}
\caption{$J(\curveparam)$ trends with different loss functions for the 2-objectives LQG.}
\end{figure}

\subsection{Linear-Quadratic Gaussian regulator}
The first case of study is a discrete-time Linear-Quadratic Gaussian regulator (LQG) with multidimensional and continuous state and action spaces~\cite{Peters2008reinf}. The LQG problem is defined by the following dynamics
\begin{align*}
s_{t+1} = A& s_t+B a_t,\quad
a_t \sim \mathcal{N}\left(K\cdot s_t, \Sigma\right)\\
&r_t = -\transpose{s_t}Qs_t - \transpose{a_t} R a_t
\end{align*}
where $s_t$ and $a_t$ are $n$-dimensional column vector ($\dstate=\daction$), $A,B,Q,R \in \realspace^{\dstate \times \dstate}$, $Q$ is a symmetric semidefinite matrix and $R$ is a symmetric positive definite matrix.
Dynamics are not coupled, that is, $A$ and $B$ are identity matrices.
The policy is Gaussian with parameters $\ppvect = vec(K)$, where $K \in \realspace^{\dstate \times \dstate}$. Finally, a constant covariance matrix $\Sigma = I$ has been chosen.

The LQG can be easily extended to account for multi-conflicting objectives. In particular, the problem of minimizing the distance from the origin w.r.t. the $i$-th axis has been taken into account, considering the cost of the action over the other axes
\begin{equation*}
\rmodel_i\left(s,a,s'\right) = -s_i^2 - \sum_{i\neq j} a_j^2.
\end{equation*}
Since the maximization of the $i$-th objective requires to have null action on the other axes, objectives are conflicting.\\
As this reward formulation violates the positiveness of matrix $R_i$, we change the reward adding an $\xi$-perturbation
\begin{equation*}
 \rmodel_i(s,a,s') = -(1-\xi)\left(s_i^2 + \sum_{i\neq j} a_j^2\right) -\xi \left(\sum_{j\neq i} s_j^2 + a_i\right),
\end{equation*}
where $\xi$ is sufficiently small. 

The values of the parameters used for all the experiments are the following ones: $\gamma = 0.9, \Sigma = I, \xi = 0.1$ and the initial state $s_0 = \transpose{[10,10]}$.

\begin{figure}[t]
\centering
\subfigure[\label{F:A_1}]{
 \includegraphics{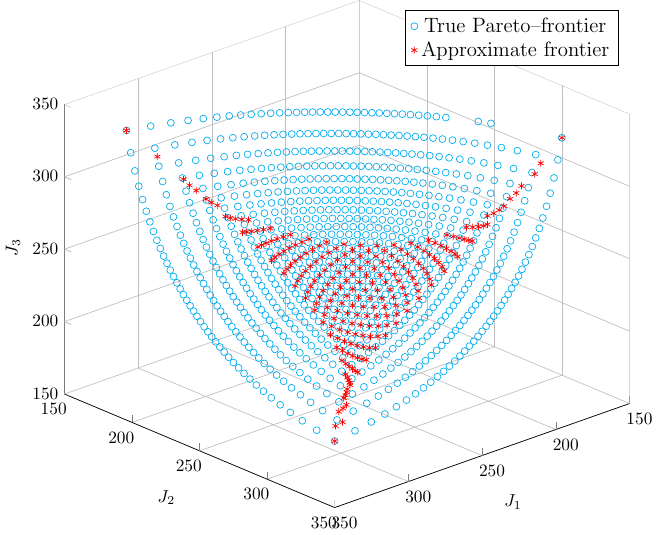}
}
\subfigure[\label{F:A_2}]{
 \includegraphics{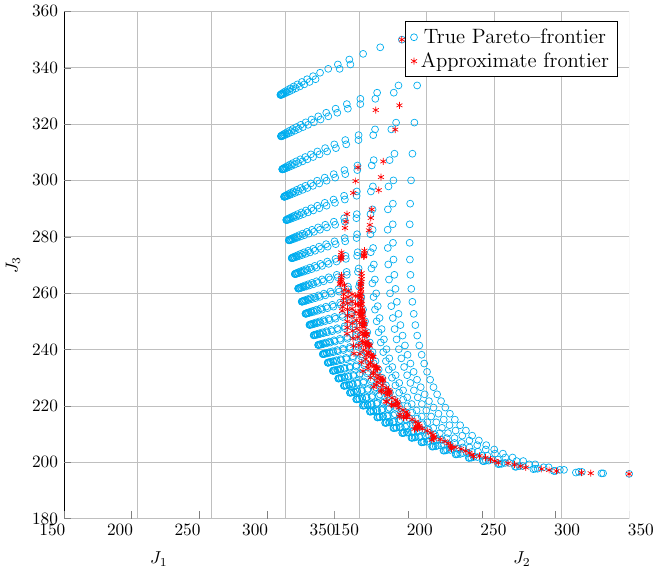}
}
\caption{Different views of the frontier obtained by \PMGA[] using $\Indf_2$ and $\Indf_1(\jvect,\mathbf{p}_u)$ for the 3-objective LGQ.}
\label{F:A}
\end{figure}

\begin{figure}[t]
\centering
\subfigure[\label{F:B_1}]{
 \includegraphics{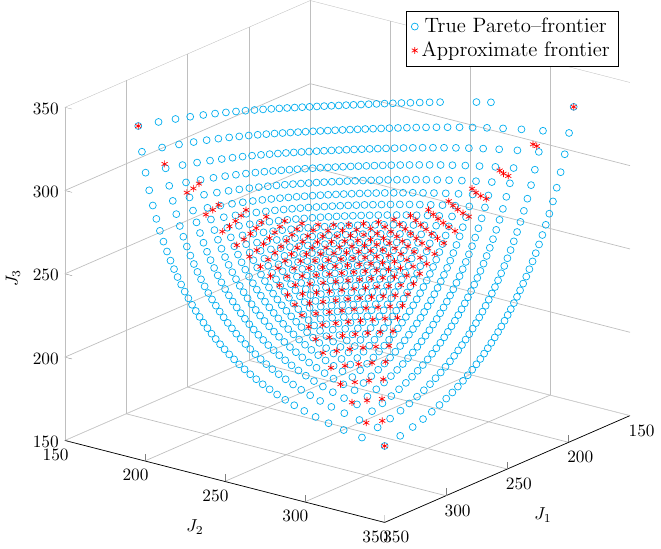}
}
\subfigure[\label{F:B_2}]{
 \includegraphics{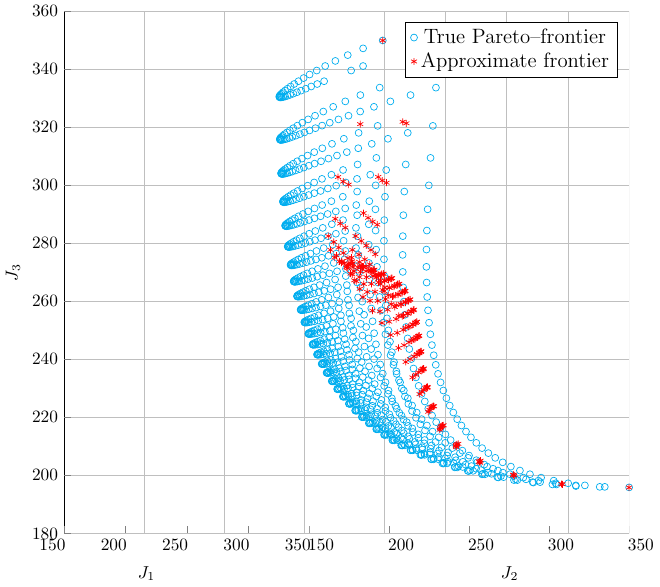}
}
\caption{Different views of the frontier obtained by \PMGA[] using normalized $\Indf_1(\jvect,\mathbf{p}_u)$ for the 3-objective LQG.}
\label{F:B}
\end{figure}

\subsubsection{2-objectives case results}
We first present the results obtained using \PMGA[] algorithm and a parametrization that is not forced to pass through the extrema of the frontier. It is the one presented in the paper and it only limits $\theta_i$ in the interval $[-1,0]$:
\begin{align*}
\theta_1 &= (1+\exp(\rho_1+\rho_2t))^{-1}
\\
\theta_2 &= (1+\exp(\rho_3+\rho_4t))^{-1}
\end{align*}

In this case using $\Indf_1$ and $\Indf_2$ the algorithm was not able to learn a good approximation of the Pareto--frontier in terms of accuracy and covering. Using utopia point as reference point for $\Indf_1$ (i.e., $\Indf_1(\jvect,\mathbf{p}_u)$) the frontier learned collapses in one point on the knee of the front. The same behaviour occurs using $\Indf_2$. Using antiutopia point as reference point for $\Indf_1$ (i.e., $\Indf_1(\jvect,\mathbf{p}_{au})$) the solutions returned are dominated and the frontier gets wider and tends to diverge from the true frontier expanding on the opposite half space (Figure \ref{F:Z-Jr} shows the divergent trend of $J(\curveparam)$). These behaviours are not unexpected, considering the definition of the loss functions, as explained in the section of the paper devoted to metrics.

The only loss function able to learn with this parametrization was $\Indf_3$. Figure \ref{F:X} (presented in the paper) shows a few iterations of the learning process using $\lambda = 2.5$ and starting from $\curveparam_0 = \transpose{[1\:2\:0\:3]}$ (the algorithm was also able to learn starting from different $\curveparam_0$). Figure \ref{F:X-Jr} shows the indicator $J(\curveparam)$ as function of the iterations. It is possible to notice that it converges to a constant value.

Other experiments were conducted using a different parametrization, forced \PMGA[] approximation to pass through the extrema of the frontier: 
\begin{align*}
\theta_1 &= (0.2403-\rho_2t^2+(0.6588+\rho_1)t)^{-1}
\\
\theta_2 &= (0.8991-\rho_2t^2+(-0.6588+\rho_2)t)^{-1}
\end{align*}

In this case, besides $\Indf_3$, also $\Indf_2$ proved to be an effective loss function and they both returned an accurate and wide approximation of the Pareto frontier (Figure \ref{F:Y}, also presented in the paper, shows the learning process starting from $\curveparam_0 = \transpose{[2\:2]}$).

$\Indf_2(\jvect,\mathbf{p}_{au})$ has still the same behaviour discussed before and the approximate frontier diverges from the true one (Figure \ref{F:Z}). This problem can be solved using the first normalization with $\beta = 0.9$ (lower $\beta$ are not enough to correct the behaviour of the loss function, while using higher $\beta$ the frontier returned is shorter and tends to be a line between the extreme points). 

$\Indf_1(\jvect,\mathbf{p}_u)$ has a similar behaviour, as the algorithm returns almost a line between the extreme points in order to reduce the frontier length. Using the first normalization with $\beta = -1.8$ such behaviour disappears and the frontier obtained has accurate solutions and guarantees a complete covering of the true Pareto frontier (its performance are thesame as $\Indf_2$ and $\Indf_3$).

The second normalization, instead, has a critical problem because of the different magnitudo between the loss function and the area of the frontier, and therefore is difficult to properly choose $\mathbf{w}$. A solution could be to ignore the constraint $w_1+w_2=1$, but in the 2-objectives case there is such a difference in the magnutudo that we were not able to find a suitable $\mathbf{w}$.

\begin{figure}[t]
\centering
\subfigure[\label{F:C}Without normalization.]{
 \includegraphics{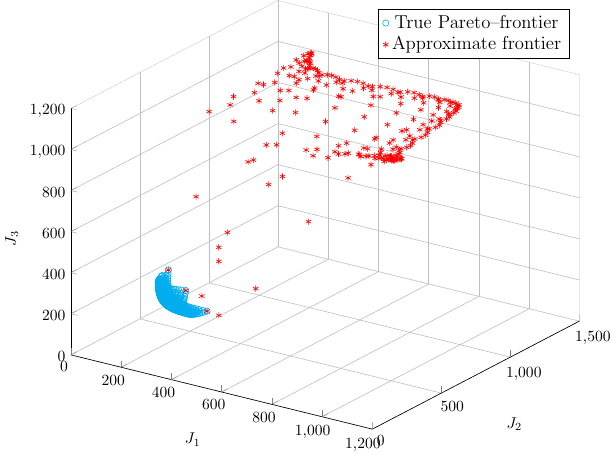}
}
\subfigure[\label{F:D}With normalization.]{
 \includegraphics{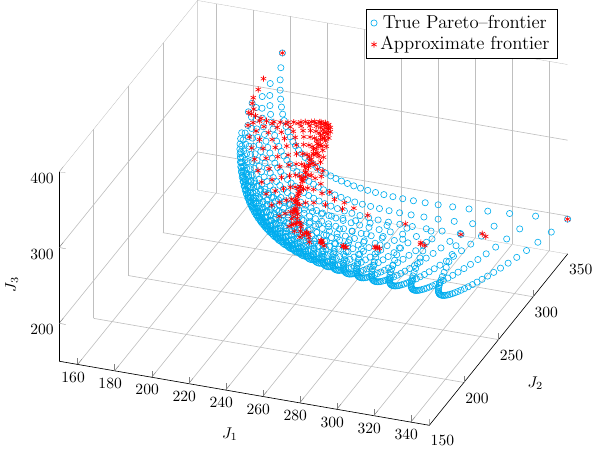}
}
\caption{Approximations of the Pareto frontier obtained by \PMGA[] using $\Indf_1(\jvect,\mathbf{p}_{au})$ for the 3-objective LQG.}
\end{figure}

\subsubsection{3-objectives case results}
We used a parametrization forced to pass through the extrema of the frontier and that limits $\theta_i$ in the interval $[-1,0]$:
\begin{align*}
\theta_1 &= -(1+\exp(a+\rho_1t_1-(b-\rho_2)t_2-\rho_1t_1^2-\rho_2t_2^2-\rho_3t_2t_1))^{-1}
\\
\theta_2 &= -(1+\exp(a-(b-\rho_4)t_1+\rho_5t_2-\rho_4t_1^2-\rho_5t_2^2-\rho_6t_1t_2))^{-1}
\\
\theta_3 &= -(1+\exp(-c+(\rho_7+b)t_1+(\rho_8+b)t_2-\rho_7t_1^2-\rho_8t_2^2-\rho_9t_1t_2))^{-1}
\end{align*}
where
\begin{equation*}
a = 1.151035476 \qquad b = 3.338299811 \qquad c = 2.187264336 \qquad \mathbf{t} \in simplex([0,1])
\end{equation*}
The initial $\curveparam_0$ is set to $\mathbf{0}$.

Figure \ref{F:A} shows the frontiers obtained using $\Indf_1(\jvect,\mathbf{p}_u)$, with and without normalization. We can clearly see that solutions tend to concentrate to the center of the frontier, in order to minimize the distance from the utopia point and the area of the frontier. Normalization is not able to correct this behaviour and the only result is to bump the frontier, slightly increasing its area. This effect seems to be indipendent from the normalization used and from the parameters $\mathbf{w}$ and $\beta$ (we tried with $1 < \beta < 6$ and 10 convex combinations uniformly spaced of $\mathbf{w}$).

Loss function $\Indf_2$ has the same behaviour and the frontiers obtained were very similar.

Figures \ref{F:C} and \ref{F:D} show the frontier obtained with $\Indf_1(\jvect,\mathbf{p}_{au})$, with and without normalization. As expected, without normalization (Figure \ref{F:C}) the algorithm tried to produce a frontier as wide as possible, in order to increase the distance from the antiutopia point. This behaviour led to dominated solutions and the learning process does not converge. Using the first normalization with $\beta = 2$ we were able to correct this behaviour, but the algorithm is still not able to cover the frontier completely (Figure \ref{F:D}). Using smaller $\beta$ the frontier was still too wide and contained dominated solutions, while higher $\beta$ led to smaller ones. The second normalization instead was ineffective. This is due, again, to the different magnitudo between the loss function and the area of the frontier, that makes the choice of $\mathbf{w}$ critical.

Finally Figure \ref{F:E_1} shows the frontier obtained using $\Indf_3$ with $\lambda = 135$. As expected, this loss function proved to be the best among the three, returning a good approximation of the Pareto frontier in terms of accuracy and covering, without using any normalization. Figure \ref{F:E_k} shows the Pareto frontier in the parameter space.

\begin{figure}[t]
\centering
\subfigure[\label{F:E_1}Frontier in the objective space.]{
 \includegraphics{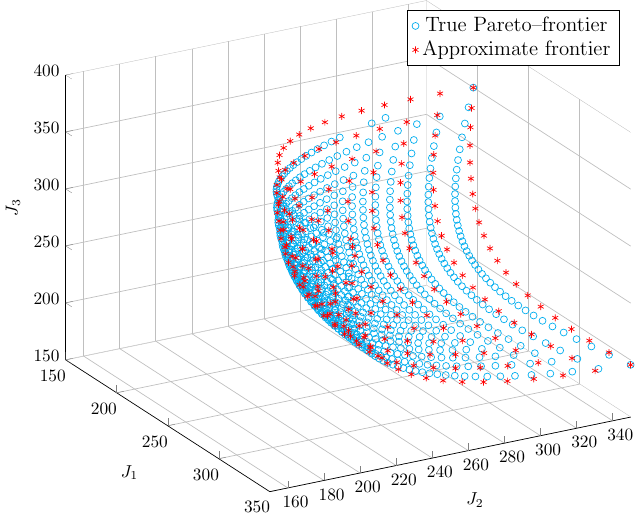}
}
\subfigure[\label{F:E_k}Frontier in the parameter space.
]{
 \includegraphics{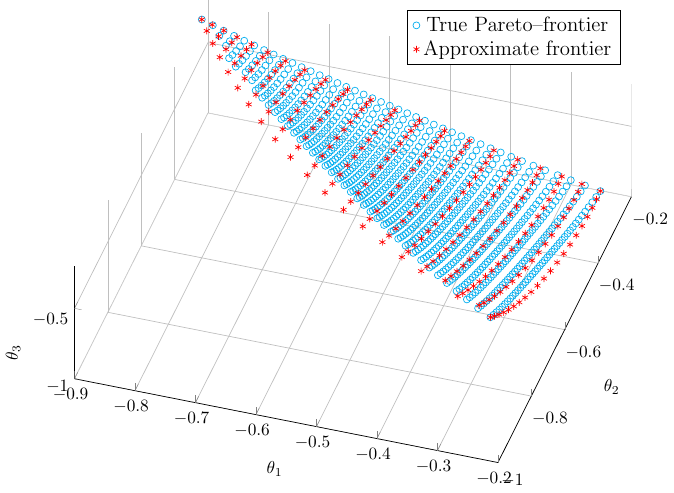}
}
\caption{Approximation of the Pareto frontier obtained by \PMGA[] using $\Indf_3$ for the 3-objective LQG.}
\label{F:E}
\end{figure}

\section{Water Reservoir}
A water reservoir can be modelled as a MOMDP with a continuous state variable $s$ representing the water volume stored in the reservoir, a continuous action $a$ that controls the water release, a state-transition model that depends also on the stochastic reservoir inflow $\epsilon$, and a set of conflicting objectives.  
For a complete description of the problem, the reader can refer to~\cite{pianosi2013tree}.

In this work we consider two objectives: flooding along the lake shores and irrigation supply. The immediate rewards are defined by
\begin{align*}
 \rmodel_1(s_t,a_t,s_{t+1}) &= -\max(h_{t+1}-\bar{h},0)\\
 \rmodel_2(s_t, a_t, s_{t+1}) &= -\max(\bar{\varrho}-\varrho_t,0)\\
\end{align*}
where $h_{t+1}=\sfrac{s_{t+1}}{S}$ is the reservoir level (in the following experiments $S=1$), $\bar{h}$ is the flooding threshold ($\bar{h}=50$), $\varrho_t = \max(\underline{a}_t, min(\bar{a}_t, a_t))$ is the release from the reservoir and $\bar{\varrho}$ is the water demand ($\bar{\varrho}=50$).
$\rmodel_1$ denotes the negative of the cost due to the flooding excess level and $\rmodel_2$ is the negative of the deficit in the water supply.

Like in the original work, the discount factor is set to $1$ for all the objectives and initial state is drawn from a finite set. However, different settings are used for learning and evaluation. In the learning phase $100$ episodes by $100$ steps are used (like in the original work), while the evaluation phase exploits $100,000$ episodes by $100$ steps.

Since the problem is continuous we exploit a Gaussian policy model $$\pi(a|s,\ppvect) = \mathcal{N}\left(\transpose{\nu(s)}\kappa,\sigma\right),$$
where $\nu : \statespace \to \realspace^{\dparam}$ are the basis functions and $\dparam = |\ppvect|$.
Since the optimal policies for the objectives are not linear in the state variable, a radial basis approximation is used:
$\nu(s) = \left[e^{-\sfrac{\norm[2]{s-c_i}}{w_i}}\right]_{i=1}^{\dparam},$
where the centres $c_i$ are placed at $0$, $50$, $120$ and $160$, and the widths are $50$, $20$, $40$ and $50$.

\begin{figure}[t]
\centering
 \includegraphics{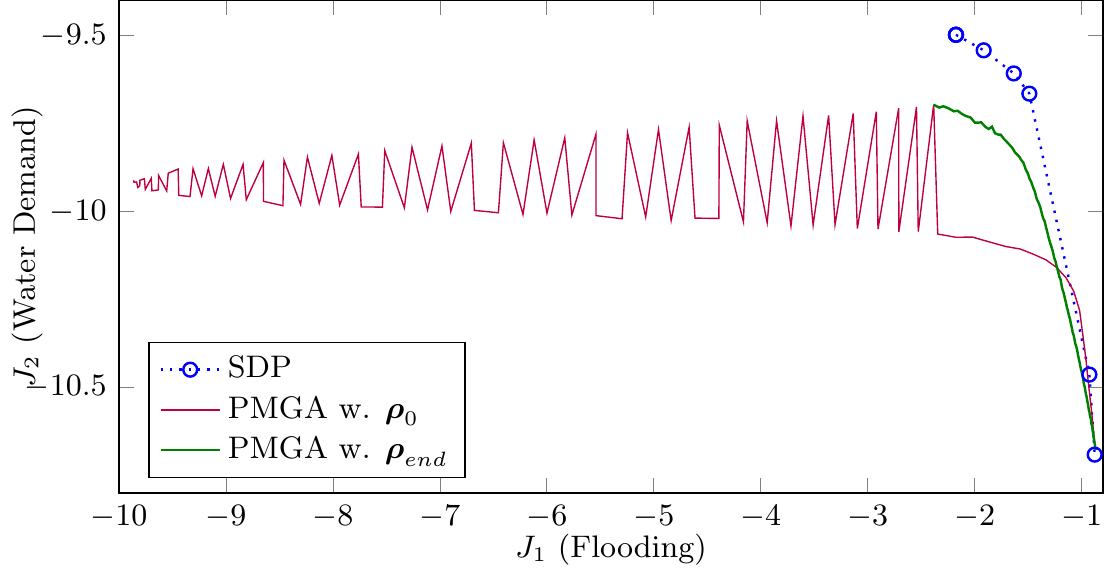}
\caption{Initial and final frontiers for a learning process with $\Indf_1(\jvect,\mathbf{p}_u)$.}
\label{F:dam2obj_utopia_P3_start}
\end{figure}

\begin{figure}[t]
\centering
\subfigure[\label{F:dam2obj_utopia_P3}]{
 \includegraphics{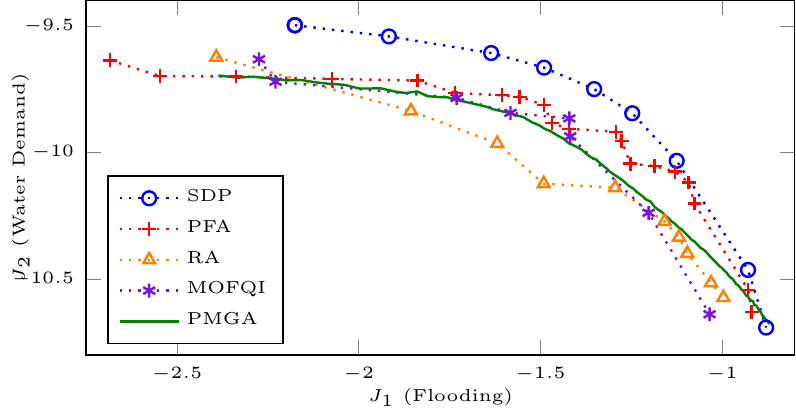}
}
\subfigure[\label{F:dam2obj_utopia_P3_Jr}]{
 \includegraphics{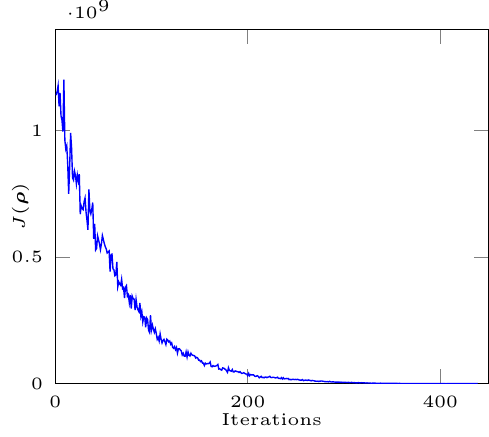}
}
\caption{Results for the water reservoir domain. Using utopia-based loss function $J(\curveparam)$ trend is convergent (on the right) and the frontier returned is comparable to the ones obtained with state-of-the-art algorithms.}
\label{F:dam}
\end{figure}

\subsubsection{Results}
We used the following parametrization, forced to pass near the estreme points of the Pareto frontier:
\begin{align*}
\theta_1 &= 61.4317 + (-11.4317 + \rho_1)t - \rho_1t^2
\\
\theta_2 &= -64.1980 + (14.1980 + \rho_2)t - \rho_2t^2
\\
\theta_3 &= 10.6159 + (-3.6159 + \rho_3)t - \rho_3t^2
\\
\theta_4 &= -22.8306 + (44.8306 + \rho_4)t - \rho_4t^2
\\
\theta_5 &= 37.8708 + (67.1292 + \rho_5)t - \rho_5t^2
\end{align*}
A constant variance $\sigma = 0.1$ has been chosen.

In order to show the capability of the approximate algorithm we have decided to test the simplest metric, that is, the utopia--based indicator. We start the learning from an arbitrary parametrization $\curveparam_0 = \mathbf{-20}$. Figure \ref{F:dam2obj_utopia_P3_start} reports the initial and the final frontiers obtained with out algorithm. We can notice that, even starting far from the true Pareto frontier, out algorithm is able to approach it, increasing covering and accuracy of the approximate frontier.

Figure \ref{F:dam2obj_utopia_P3} reports the final frontier obtained with different algorithms. The approximation obtained by our algorithm is comparable to the other results, however, our approach is able to produce a continuous frontier approximation.

It is important to notice that, due to the fact that the transition function of the domain limits the action in the range of admissible values ($a_t \in [\underline{a}_t,\bar{a}_t]$), there are infinite policies with equal performance that allow the agent to release more than the reservoir level or less than zero. To overcome this problem~\cite{Parisi2014morl} introduces a penalty term $p$ in the reward ($p= -\max(a_t - \bar{a}_t, \underline{a}_t - a_t)$) during the learning phase. With our approach this modification was unnecessary, as the algorithm is able to learn without the penalty. We also tried adding it during the learning phase, but the frontier returned was exactly the same.

\section{Metrics $\Indf_3$ tuning}\label{S:tuning}
In this Section we want to examine more deeply the tuning of mixed metric parameters, in order to provide the reader better insights for a correct use of such metric. \PMGA[] performance, indeed, strongly depends on the indicator used and, thereby, their setting is critical. To be more precise, mixed metric, which obtained the best approximate Pareto--frontiers in the experiments, includes a trade-off between accuracy and covering, expressed by some parameters.

The indicator we are going to analyze is
\begin{equation*}
 \Indf_3(\jvect) = \Indf_1(\jvect,\mathbf{p}_{AU}) \cdot w(\jvect)
\end{equation*}
where $w(\jvect)$ is a penalization term, i.e., it is a monotonic function that decreases as $\Indf_2(\jvect)$ increases. In the previous Sections we proposed $w(\jvect) = 1 - \lambda \Indf_2(\jvect)$. In this way we take advantage of the expansive behavior of the antiutopia--based indicator and the accuracy of the optimality--based indicator $\Indf_2$. In this Section we are going to study the performance of this metric on varying $\lambda$, proposing a simple tuning process. The idea is to set $\lambda$ to an initial value (for example 1) and then increase (or dicrease) it if the approximate frontier contains dominated solutions (or is not large enough).
Figure \ref{F:tuning_mix1} shows different approximate frontiers obtained with different $\lambda$. Starting with $\lambda = 1$ the indicator behaves like $\Indf_1(\jvect,\mathbf{p}_{AU})$, meaning that $\lambda$ was too small. Using $\lambda = 1.5$ (Figure \ref{F:tuning_mix1_a}) the algorithm converges but the approximate frontier still contains dominated solutions. Increasing $\lambda$ to 1.5 (Figure \ref{F:tuning_mix1_b}) dominated solutions disappear. Finally, with $\lambda = 2.5$ (Figure \ref{F:tuning_mix1_c}) the approximate frontier becomes shorter and Pareto--optimal solutions are discarded, meaning that we increased $\lambda$ too much.

\begin{figure}[t]
\centering
\subfigure[\label{F:tuning_mix1_a}$\lambda = 1.5$]{
 \includegraphics{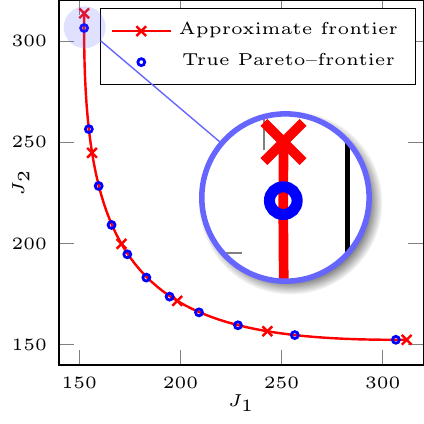}
}
\subfigure[\label{F:tuning_mix1_b}$\lambda = 2$]{
 \includegraphics{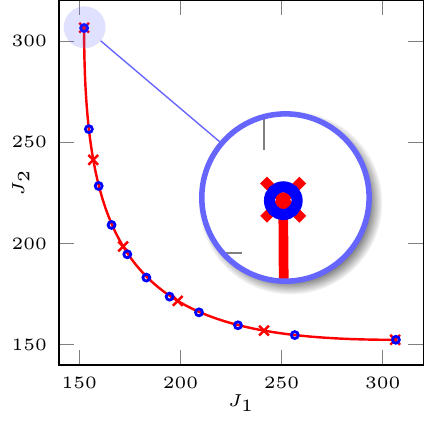}
}
\subfigure[\label{F:tuning_mix1_c}$\lambda = 2.5$]{
 \includegraphics{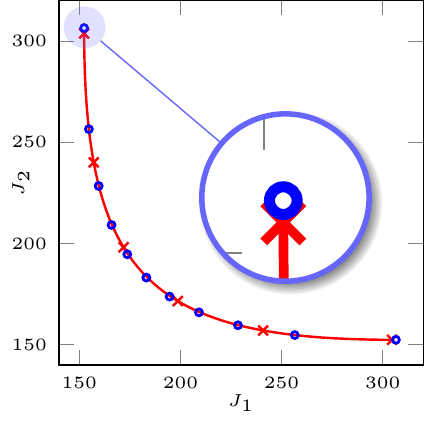}
}
\caption{Approximate frontiers learned by \PMGA[] using $\Indf_3$ on varying $\lambda$. Figure \subref{F:tuning_mix1_a} has dominated solutions and \subref{F:tuning_mix1_c} is not wide enough. On the contrary, \subref{F:tuning_mix1_b} achieves both accuracy and covering.}
\label{F:tuning_mix1}
\end{figure}

\bibliography{./rlbibdb}
\bibliographystyle{unsrt}
\end{document}